\documentclass[letterpaper, 10 pt, conference]{ieeeconf}  

\IEEEoverridecommandlockouts                              

\usepackage{cite}
\usepackage{amsmath,amssymb,amsfonts}
\usepackage{graphicx}
\usepackage{textcomp}
\usepackage{xcolor}
\usepackage{epstopdf}
\usepackage{epsfig} 
\usepackage{times} 
\usepackage{tabularx}
\usepackage{dcolumn}
\usepackage{color} 
\usepackage{algorithm}
\usepackage[algo2e,linesnumbered,vlined,ruled]{algorithm2e} 
\usepackage{balance}
\usepackage{accents}
\usepackage{resizegather}
\usepackage{theorem} 
\usepackage{mathtools}
\usepackage{multirow}
\usepackage{amsmath}
\DeclareMathOperator*{\argmin}{argmin}
\usepackage{bm}
\usepackage{wrapfig}
\usepackage{graphicx}
\usepackage{subcaption}
\usepackage{caption}
\usepackage{mdframed}
\usepackage{algpseudocode}
\usepackage{hyperref}

\usepackage{caption}
\newtheorem{definition}{\bf{Definition}}

\newtheorem{problem}{\bf{Problem}}
\newtheorem{remark}{\bf{Remark}}
\newtheorem{theorem}{\bf{Theorem}}

\newcommand{\LTLUNTIL}{\ensuremath{U}}
\newcommand{\LTLEVENTUALLY}{\ensuremath{ F}}
\newcommand{\LTLALWAYS}{\ensuremath{ G }}

\newcommand{\notltl}{\neg}
\newcommand{\andltl}{\wedge}
\newcommand{\orltl}{\vee}

\newcommand{\remainingtime}{\tau}

\SetKw{KwAnd}{and}

\title{\LARGE \bf
Shielded Reinforcement Learning Under Dynamic Temporal Logic Constraints

\author{Sad{\i}k Bera Y\"{u}ksel, Ali Tevfik Buyukkocak, and Derya Aksaray}
\thanks{S.B. Y\"{u}ksel is a PhD student in the Department of Electrical and Computer Engineering at Northeastern University.}%
\thanks{A.T. Buyukkocak is a Postdoctoral Researcher in the Department of Electrical and Computer Engineering at Northeastern University.}
\thanks{D. Aksaray is an Assistant Professor in the Department of Electrical and Computer Engineering at Northeastern University.}
}

\begin{document}

\maketitle
\thispagestyle{empty}
\pagestyle{empty}

\begin{abstract}
Reinforcement Learning (RL) has shown promise in various robotics applications, yet its deployment on real systems is still limited due to safety and operational constraints. The safe RL field has gained considerable attention in recent years, which focuses on imposing safety constraints throughout the learning process. However, real systems often require more complex constraints than just safety, such as periodic recharging or time-bounded visits to specific regions. Imposing such spatio-temporal tasks during learning still remains a challenge. Signal Temporal Logic (STL) is a formal language for specifying temporal properties of real-valued signals and provides a way to express such complex tasks. In this paper, we propose a framework that leverages sequential control barrier functions and model-free RL to ensure that the given STL tasks are satisfied throughout the learning process. Our method extends beyond traditional safety constraints by enforcing rich STL specifications, which can involve visits to dynamic targets with unknown trajectories. We also demonstrate the effectiveness of our framework through various simulations.
\end{abstract}


\section{INTRODUCTION}
Reinforcement Learning (RL) has achieved remarkable success in many robotics applications, such as manipulation, locomotion, and navigation \cite{RLinRoboticsSurvey}. An RL agent learns to perform a task in the most optimal way by interacting with an environment and receiving feedback in the form of rewards. The learning process typically involves exploring a large number of state-action pairs to learn the optimal policy that maximizes the expected reward. This enables the agent to learn different complex tasks in unknown environments without requiring an explicit model of the system dynamics.

Despite its success in policy generation, applications of RL to real robotic systems are still limited. Random exploration during learning might lead to unsafe behaviors, which can harm the robot or its surroundings. These safety concerns, e.g., collision avoidance, make the real-world training impractical for many applications. As a result, constrained/safe RL has become an active research area (e.g., \cite{garcia2015comprehensive}). 
Existing approaches can be broadly categorized into two groups: methods achieving a safe policy by the end of the learning (e.g., \cite{achiam2017constrainedpolicyoptimization, wachi2018safe}), and methods treating safety as a hard constraint by enforcing it during the learning (e.g., \cite{cheng2019endtoendsafereinforcementlearning, emam2022safereinforcementlearningusing}).

However, real-world training may require several operational constraints in addition to safety. Real systems are subject to practical limitations that can interrupt learning and require human intervention. For instance, robots may need periodic battery recharging or resetting between episodes, which can interrupt autonomous learning. These requirements can be defined as spatio-temporal tasks that must be satisfied to continue the learning process. In the literature, a common approach to encode such tasks is temporal logic (TL), a formal language for specifying and reasoning about complex behaviors \cite{baier2008principles}. In this paper, we focus on Signal Temporal Logic (STL) \cite{maler2004monitoring}, which is an expressive language to specify the temporal properties of real-valued signals.

Several works explored integrating temporal logic into RL. 
In \cite{aksaray2016q}, reward shaping is used to learn a policy that satisfies the STL specification, but the specification is not enforced during training. In \cite{li2019temporallogicguidedsafe}, syntactically co-safe Truncated Linear Temporal Logic (scTLTL) is used for reward shaping, and safety is enforced during learning by using control barrier functions (CBFs). However, this method only considers hard constraints in the form of safety (i.e., confining the agent to a safe region) and cannot express explicit time windows for the tasks due to the limited expressiveness of scTLTL. The work in \cite{safeRLforSTL} aims to learn the optimal policy to satisfy the given STL tasks, but focuses only on safety constraints during learning 
and does not leverage the full STL syntax. Moreover, \cite{aksaray2021probabilistically}, \cite{lin2023reinforcement}, and \cite{lin2024probabilistic} enforce Bounded TL specifications and provide probabilistic satisfaction guarantees during learning. However, the approach is limited to discrete state spaces. Furthermore, none of these works consider time-varying specifications involving dynamic target sets. The work in \cite{BUYUKKOCAK2024104681} addresses this gap by proposing sequential CBFs to satisfy STL tasks with dynamic targets, but it is not formulated in an RL setting. These limitations are critical for real-world learning, where constraint violations during training can interrupt learning, require human intervention, and reduce autonomy.

The contributions of this paper are summarized as follows:
\begin{itemize}
    \item We propose a model-free RL framework that enforces STL specifications throughout training to ensure that complex spatio-temporal constraints are satisfied during learning rather than only by the final learned policy.
    \item The framework supports a rich class of STL tasks, including those with dynamic targets with unknown trajectories, and enforces realistic operational constraints, such as periodic visits to a charger or resetting to a desired initial state between learning episodes.
    \item We demonstrate the effectiveness of the proposed framework through simulation studies and verify its ability to enforce STL constraints during learning.
\end{itemize}

\section{PRELIMINARIES}
\subsection{Signal Temporal Logic}
In this paper, we specify the desired system behaviors using Signal Temporal Logic (STL).

\begin{definition}[Signal Temporal Logic]Signal Temporal Logic (STL) \cite{STL} is a formal language used to specify temporal properties of a system over real-valued signals. In this paper, we use the following STL syntax:
\begin{flalign}
\mathnormal{\Phi} &::= \mathnormal{\Phi_1} \land \mathnormal{\Phi_2} \mid \mathnormal{\Phi}_1 \lor \mathnormal{\Phi}_2 \mid \LTLEVENTUALLY_{[a,b]}\phi \mid \LTLALWAYS_{[a,b]}\phi \mid \varphi_1 \LTLUNTIL_{[a,b]} \varphi_2, \notag \\
\phi &::= \varphi \mid \neg \phi \mid \LTLEVENTUALLY_{[c,d]} \varphi \mid \LTLALWAYS_{[c,d]} \varphi , \label{eq:STL_syntax} \\
\varphi &::= \mu \mid \neg \varphi \mid \varphi_1 \land \varphi_2 \mid \varphi_1 \lor \varphi_2 , \notag
\end{flalign}
where $a,b,c,d \in \mathbb{R}_{\geq 0}$ are finite non-negative time bounds with $b \geq a$ and $d \geq c$;  $\LTLEVENTUALLY$, $\LTLALWAYS$, $\LTLUNTIL$ denote the finally (eventually), globally (always) and until temporal operators, respectively; $\notltl$, $\andltl$, $\orltl$ denote the negation, conjunction and disjunction Boolean operators, respectively; $\mathnormal{\Phi}, \phi, \varphi$ are STL specifications, and $\mu$ is a predicate of the form $p(\mathbf{x}(t)) \sim 0$, where $\mathbf{x}: \mathbb{R}_{\geq 0} \rightarrow \mathbb{R}$ is a real-valued signal, $p:\mathbb{R}^n \rightarrow \mathbb{R}$ is a function, and $\sim \in \{\geq, \leq\}$ is a relational operator.
\end{definition}

For any signal $\mathbf{x}$, let $\mathbf{x}(t)$ denote its value at time $t$ and $(\mathbf{x},t)$ denote the part of the signal starting from $t$. The satisfaction of an STL specification over $(\mathbf{x},t)$ is defined as follows:
\begin{flalign*}
&(\mathbf{x}, t) \models \mu \iff p(\mathbf{x}(t)) \sim 0, & \\
&(\mathbf{x}, t) \models \neg \mu \iff \neg((\mathbf{x}, t) \models \mu), & \\
&(\mathbf{x}, t) \models \mathnormal{\Phi_1} \land \mathnormal{\Phi_2} \iff (\mathbf{x}, t) \models \mathnormal{\Phi_1} \text{ and } (\mathbf{x}, t) \models \mathnormal{\Phi_2}, & \\
&(\mathbf{x}, t) \models \mathnormal{\Phi_1} \lor \mathnormal{\Phi_2} \iff (\mathbf{x}, t) \models \mathnormal{\Phi_1} \text{ or } (\mathbf{x}, t) \models \mathnormal{\Phi_2}, & \\
&(\mathbf{x}, t) \models \mathnormal{\Phi_1} U_{[a,b]} \mathnormal{\Phi_2} \iff \exists t' \in [t+a, t+b],\ (\mathbf{x}, t') \models \mathnormal{\Phi_2} & \\
&\quad \text{and } \forall t'' \in [t, t'],\ (\mathbf{x}, t'') \models \mathnormal{\Phi_1}, & \\
&(\mathbf{x}, t) \models \LTLALWAYS_{[a,b]} \mathnormal{\Phi} \iff \forall t' \in [t+a, t+b],\ (\mathbf{x}, t') \models \mathnormal{\Phi}, & \\
&(\mathbf{x}, t) \models \LTLEVENTUALLY_{[a,b]} \mathnormal{\Phi} \iff \exists t' \in [t+a, t+b],\ (\mathbf{x}, t') \models \mathnormal{\Phi}. &
\end{flalign*}

The operator $\LTLEVENTUALLY_{[a,b]}\mathnormal{\Phi}$ specifies that $\mathnormal{\Phi}$ should be satisfied at least once in the interval $[t+a, t+b]$, while $\LTLALWAYS_{[a,b]}\mathnormal{\Phi}$ specifies that $\mathnormal{\Phi}$ must hold at all times within the interval $[t+a, t+b]$. On the other hand, the $\mathnormal{\Phi_1} U_{[a,b]} \mathnormal{\Phi_2}$ operator specifies that $\mathnormal{\Phi_1}$ must hold until $\mathnormal{\Phi_2}$ becomes true in the interval $[t+a, t+b]$.
The horizon of an STL specification $\mathnormal{\Phi}$, denoted as $hrz(\mathnormal{\Phi})$, is defined as the number of future samples required to determine the truth value of the specification \cite{dokhanchi2014onlinemonitoringtemporallogic}. 

In this paper, we consider STL specifications $\mathnormal{\Phi}$ composed of multiple tasks $\mathnormal{\Phi}_i$ in conjunction as follows: 
\begin{equation}
\mathnormal{\Phi} := \mathnormal{\Phi_1} \land \mathnormal{\Phi_2} \land \cdots \land \mathnormal{\Phi_k}.
\label{overallSTL}
\end{equation}

Each task $\mathnormal{\Phi}_i$ includes one or more temporal operators applied on inner specification $\varphi_i$ as described in the syntax (\ref{eq:STL_syntax}). For instance, $\mathnormal{\Phi}_i = \LTLALWAYS_{[a,b]} \LTLEVENTUALLY_{[c,d]} \varphi_i$ or $\mathnormal{\Phi}_j = \LTLEVENTUALLY_{[a,b]} (\varphi_{j,1} \lor \varphi_{j,2})$, where $\varphi_j = \varphi_{j,1} \lor \varphi_{j,2}$ is the inner specification associated with task $\mathnormal{\Phi}_j$.

\subsection{Time-Varying (Zeroing) Control Barrier Functions}
Consider a control affine system of the form 
\begin{equation}
\Dot{x}=f(x)+g(x)u,
\end{equation}
where $f:\mathbb{R}^n\to\mathbb{R}^n$ and $g:\mathbb{R}^n\to\mathbb{R}^{n\times m}$ are locally Lipschitz continuous, $x \in \mathcal{X}\subseteq\mathbb{R}^n$ is the state and $u \in \mathcal{U}\subseteq\mathbb{R}^m$ is the control. Control barrier functions (CBFs) are used to render a desired set of states forward invariant \cite{Ames_2017}. That is, let $\boldsymbol{h}_\text{CBF}:\mathcal{X}\to\mathbb{R}$ be a differentiable function and define the zero superlevel set $\mathcal{C}=\{x \in \mathcal{X}\ |\ \boldsymbol{h}_\text{CBF}(x)\geq0 \}$. If the system starts in the zero superlevel set, i.e., $x(0) \in \mathcal{C}$, then CBF conditions ensure that $x(t)$ remains in $\mathcal{C}$ for all $t \geq 0$.

This concept is extended to time-varying sets by introducing time-varying CBFs \cite{Lindemann2019}. Specifically, a differentiable function $\boldsymbol{b}(x,t)\!:\!\mathcal{X}\!\times\! \mathbb{R}_{\geq0}\!\to\!\mathbb{R}$ is defined as a time-varying CBF and the corresponding time-varying set $\mathcal{C}(t)=\{x\in\mathcal{X}\ |\ \boldsymbol{b}(x,t)\geq 0\}$ is rendered forward invariant by the input $u \in \mathcal{U}$, if $x(0)\in \mathcal{C}(0)$ and there exists a locally Lipschitz continuous class $\mathcal{K}$ function $\alpha:\mathbb{R}_{\geq 0}\to\mathbb{R}_{\geq0}$ such that 

\begin{equation}\label{eq:TVZCBF}
\small
\sup_{u \in\mathcal{U}} \frac{\partial\boldsymbol{b}(x,t)}{\partial x}^T(f(x)+g(x)u)+\frac{\partial\boldsymbol{b}(x,t)}{\partial t}+\alpha(\boldsymbol{b}(x,t))\geq0.
\end{equation}

Such time-varying CBFs are useful for enforcing STL tasks as they can encode temporal properties. For example, we can steer an agent to a target region within a specific time bound by using time-varying CBFs.

\section{PROBLEM STATEMENT}
We consider a model-free RL setting, where the agent tries to learn the optimal policy through interactions with an unknown environment. The environment is modeled as a finite-horizon Markov Decision Process (MDP), denoted as $\mathcal{M} = (\mathcal{X}, \mathcal{U}, p, r)$ where 
\begin{itemize}
\item $\mathcal{X} \subseteq \mathbb{R}^n$ is the continuous state space;

\item $\mathcal{U} \subseteq \mathbb{R}^m$ is the continuous input space;

\item $p:\mathcal{X} \times \mathcal{U} \times \mathcal{X} \rightarrow [0,1]$ is the state transition function; 

\item $r: \mathcal{X} \times \mathcal{U} \times \mathcal{X} \rightarrow \mathbb{R}$ is the reward function.
\end{itemize}
The agent interacts with the environment over episodes of fixed length $T$. At each time step, it selects an action according to a policy $\pi$, which may be deterministic $\pi: \mathcal{X} \rightarrow \mathcal{U}$, or stochastic $\pi: \mathcal{X} \times \mathcal{U} \rightarrow [0,1]$.

The objective of the agent is to learn the optimal policy that maximizes the total expected episodic reward while ensuring that the \emph{state trajectory generated in each episode} satisfies a desired STL specification. Unlike prior work that focuses on learning how to satisfy the given STL specification through reward shaping (e.g.,~\cite{safeRLforSTL,aksaray2016q}),  the reward function in our problem setting is only associated with the main objective (e.g., minimizing energy, maximizing information gain, or task efficiency) and decoupled from the STL specification. Thus, we treat the STL specification as a hard constraint and aim to satisfy it \emph{throughout the entire learning process}. 

Enforcing STL specifications during training serves as a formal certificate of desired behavior in real-world applications.  
For example, in an autonomous driving scenario, the STL specifications may require the agent to keep its speed below the speed limit and maintain a safe following distance from the other cars, while the reward function encourages fast traversal to the destination. Similarly, in a UAV surveillance scenario, the agent may receive rewards for flying over the surveillance areas, while an STL specification mandates periodic visits to a charging station during the learning so that the agent can continue the learning process without human intervention. 
In such scenarios, enforcing STL constraints throughout learning is essential as violations may disrupt learning (e.g., accidents, undesirable behavior).


In many practical scenarios, STL specifications may involve visits to dynamic target sets whose locations vary over time. For instance, in the previous UAV surveillance scenario, the charging station might be mounted on a moving ground vehicle. Dynamic target requirements can be encoded using STL formulas by imposing temporal constraints on visiting dynamic targets. However, introducing dynamic targets into the learning framework is challenging, especially when their future trajectories are unknown to the agent. In such cases, the agent must reason about the future locations of the target and choose its actions accordingly during learning. 

In our problem setting, we assume the agent has single integrator dynamics of the form:
\begin{equation}
    \dot{x} = u + d, \quad \|u\| \leq u_{\max}, \quad \|d\| \leq d_{\max},
\label{agent_dynamics}
\end{equation}
where $u$ is the control input and $d$ is an unknown but bounded disturbance with known bound $d_{\max}$. 
Each target $i$ is also assumed to follow single integrator dynamics:
\begin{equation}
    \dot{x}_i^\text{target} = u_i^\text{target}, \quad \|u_i^\text{target}\| \leq u_{i,\max}^\text{target},
\label{target_dynamics}
\end{equation}
where the control input $u_i^\text{target}$ is unknown but bounded by a known constant $u_{i,\max}^\text{target}$. This abstraction captures realistic movement while allowing the agent to estimate the future states of the targets conservatively. Accordingly, we formulate the following problem:

\begin{problem}
Given a finite horizon MDP $\mathcal{M} = (\mathcal{X}, \mathcal{U}, p, r)$ where the unknown state transition function $p$ is induced by the dynamics in (\ref{agent_dynamics}) and the reward function $r$ is unknown,
 and an STL specification $\mathnormal{\Phi}$ which incorporates time-varying targets governed by the dynamics as in (\ref{target_dynamics}), find the optimal policy that maximizes the total expected episodic reward
\begin{equation}
   \pi^* = \arg\max_{\pi} \mathbb{E}^{\pi}\left[ \sum_{t=0}^{T} \gamma^t r_t \right],
\end{equation}
such that, for every episode $j$ in the learning
\begin{equation}
    (\mathbf{x}_j, 0) \models \mathnormal{\Phi} , \quad \forall j\geq 0,
\end{equation}
where $0<\gamma \leq 1$ is the discount factor, $\mathbf{x}_j : \{0,1,\dots,T\} \rightarrow \mathcal{X}$ is the signal representing the state trajectory in episode $j$  over the finite episode horizon $T \geq hrz(\Phi)$.


\label{problem}
\end{problem}

\section{PROPOSED SOLUTION}
In this section, we first discuss the challenges of enforcing STL tasks during learning. Next, we review the sequential CBF and explain how they are used to satisfy STL tasks involving visits to dynamic target sets. Finally, we explain our proposed framework that incorporates sequential CBFs into a model-free RL setting and introduce the main algorithm enforcing the satisfaction of STL tasks throughout the learning process. An overview of our proposed framework is illustrated in Fig. \ref{fig:framework}.


\begin{figure}[!htb]
    \centering
    \includegraphics[width=0.9\linewidth]{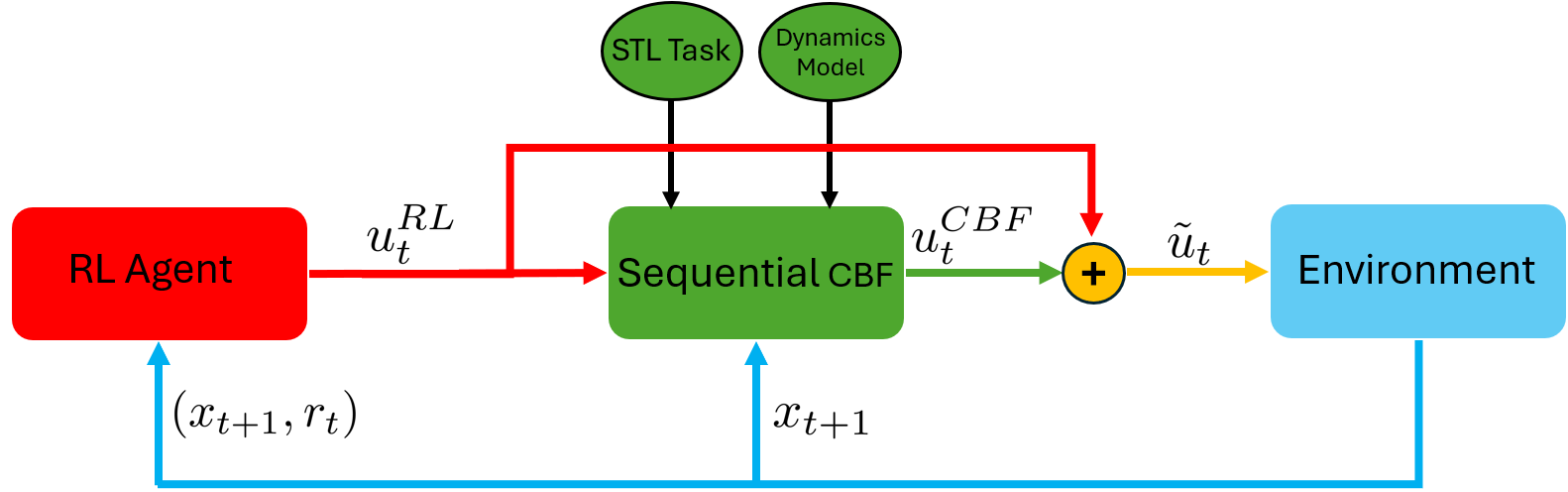}
    \captionsetup{font=footnotesize, labelfont=footnotesize}
    \caption{Overview of the proposed framework. At each time step, the RL agent outputs an unconstrained control input $u_t^{RL}$, which is then adjusted by a corrective input $u_t^{CBF}$ computed using sequential CBFs to satisfy the given STL task. The final control $\tilde{u}_t$ is applied to the system, and the resulting next state and reward $(x_{t+1}, r_t)$ are fed back to the RL agent.}
    \label{fig:framework}
\end{figure}
\vspace{-2mm}
\subsection{Reinforcement Learning Under STL Constraints}
\label{subsec:RL_STL}
Encoding STL specifications as hard constraints in RL settings introduces several challenges. First, incorporating the STL constraint into the reward function does not guarantee its satisfaction during the early stages of learning. Since the goal of RL is to maximize reward through trial and error, the satisfaction of the STL specification is eventually achieved. As a result, violations can occur throughout the training phase, which is undesirable in both simulated and real-world settings.
Second, a common approach to enforce hard constraints during learning is by using CBFs (e.g., \cite{emam2022safereinforcementlearningusing,cheng2019endtoendsafereinforcementlearning, li2019temporallogicguidedsafe,safeRLforSTL}). Standard CBFs require defining a safe set and confining the agent to that set through control corrections. The safe set is stationary in the case of safety constraints (e.g., globally avoiding obstacles), or it can shrink when enforcing reachability constraints (e.g., eventually visiting a region). However, when STL specifications involve nested operators, disjunctions, or dynamic requirements that cause the safe set to move in the state-space, standard CBFs (including time-varying variants) become insufficient. While sequential CBFs \cite{buyukkocak2022control,BUYUKKOCAK2024104681} can encode richer STL specifications with dynamic predicates, their integration into an RL setting has not been explored yet. The original sequential CBF formulation does not account for system disturbances or the presence of secondary objectives such as balancing STL satisfaction with the need for exploration during learning. These aspects are essential in RL settings, and we address them to extend the use of sequential CBFs in RL problems.

\subsection{Sequential CBFs for Dynamic STL Specifications}
Sequential CBFs were first introduced in \cite{BUYUKKOCAK2024104681} and \cite{buyukkocak2022control} to solve motion planning problems with dynamic STL objectives. In this paper, we build on the framework of sequential CBFs and enable RL that enforces rich STL constraints during training.
In the sequential CBF formulation, the predicates of the desired STL formula are related to the dynamic target locations. As discussed in \eqref{target_dynamics}, we assume that the exact target locations are not known, but their control inputs are bounded by a known constant. 
Accordingly, the worst-case future distances to/between target regions can be defined as follows:
\vspace{-2mm}
\begin{definition}[Worst-Case Distance to Target] Let $x \in \mathcal{X}$ be the agent's state, $\mathcal{P} \subseteq \mathcal{X}$ be a compact set corresponding to the satisfaction of an inner specification $\varphi$ as defined in (\ref{eq:STL_syntax}) (i.e., $\mathcal{P}:= \{x \in \mathcal{X} \mid x \models \varphi \}$), and $\remainingtime \in \mathbb{R}_{\geq 0}$ be the remaining time to satisfy $\varphi$. Then the worst-case distance is defined as:
\begin{equation}
    Dist_{wc}(x,\mathcal{P},\tau) := Dist(x, \mathcal{P}) + u_{\mathcal{P},\max} \cdot \remainingtime,
    \label{distance_agent}
\end{equation}
where $Dist(x, \mathcal{P})$ is the signed distance from $x$ to $\mathcal{P}$. 
\end{definition}
Additionally, the worst-case pairwise distance between two ordered targets is defined as follows:
\begin{equation}
    Dist_{wc}(\mathcal{P}_i,\mathcal{P}_j, \remainingtime_i, \remainingtime_j) := \sup\limits_{x \in \mathcal{P}_i} \inf\limits_{x^* \in \mathcal{P}_j}\|x-x^*\| + u_{\mathcal{P}_i, \max} \cdot \remainingtime_i + u_{\mathcal{P}_j, \max} \cdot \remainingtime_j,
\label{distance_pairwise}
\end{equation}


To determine the sequence in which the dynamic targets associated with STL tasks should be visited, we use a variant of the sequence generation algorithm proposed in \cite{BUYUKKOCAK2024104681}. The original algorithm generates a feasible ordering of tasks based on their remaining time and worst-case distances as defined in (\ref{distance_agent}) and (\ref{distance_pairwise}). By considering the worst-case distances, it ensures that the generated sequence is feasible under the given dynamic targets, even without knowledge of their trajectories. Unlike \cite{BUYUKKOCAK2024104681}, we introduce a minor modification to the algorithm to better align with the RL setting. While sequences involving frequent visits to targets in periodic tasks still constitute correct solutions, they leave less time for exploration. Therefore, we prioritize candidate sequences with fewer repetitions for periodic STL tasks.


After generating a feasible sequence $\mathcal{S}$ of $k$ tasks, we construct the following sequential CBF candidate for each task in the sequence:

\begin{flalign}
& \bm{b}_i(x, t) := \bar{\bm{\remainingtime}}_{\mathcal{S}_i}(t)
- \frac{Dist_{wc}\left(x, \mathcal{P}_{\mathcal{S}_1}, \bar{\bm{\remainingtime}}_{\mathcal{S}_1}(t)\right)}{u_{\max} - d_{\max}} \label{sequentialCBF} \\
&- \frac{\sum_{l=2}^{i} Dist_{wc}\left(\mathcal{P}_{\mathcal{S}_{l-1}}, \mathcal{P}_{\mathcal{S}_l}, \bar{\bm{\remainingtime}}_{\mathcal{S}_{l-1}}(t), \bar{\bm{\remainingtime}}_{\mathcal{S}_l}(t)\right)}{u_{\max} - d_{\max}}, \ \forall i \in \mathcal{S}, \notag
\end{flalign}
where $\bar{\bm{\remainingtime}}_{\mathcal{S}_i}(t) = \min_{j \in \{i,\dots,|\mathcal{S}|\}} \bm{\remainingtime}_{\mathcal{S}_j}(t)$ denotes the actual remaining time available for completing task $\mathcal{S}_i$ while preserving feasibility of all subsequent tasks in the sequence. In other words, task $\mathcal{S}_i$ must be completed early enough so that the remaining tasks can still be satisfied before their respective deadlines. Here, $\bm{\remainingtime}_{\mathcal{S}_j}: \mathbb{R}_{\geq 0} \rightarrow \mathbb{R}_{\geq 0}$ is the remaining time to satisfy the $j^{th}$ task in the sequence $\mathcal{S}$, defined as $\bm{\remainingtime}_{\mathcal{S}_j}(t) = \bm{\remainingtime}_{\mathcal{S}_j}(0)-t$. Note that \eqref{sequentialCBF} is different from the original formulation in \cite{BUYUKKOCAK2024104681} by subtracting the maximum disturbance bound $d_{\max}$ from the agent's maximum input $u_{\max}$ in the denominator. This modification ensures that the STL task is satisfiable under the worst-case disturbance.


To ensure that the system satisfies the given STL tasks in the specified order, it must be governed by the most time-critical CBF among all candidate CBFs. We determine the most time-critical CBF via,
\begin{equation}
    \bm{b}(x,t) = \min_{i \in \mathcal{S}} \bm{b}_i(x,t).
\label{critical_CBF}
\end{equation}

Extending this framework to other dynamical systems requires modifying the reachability checks in \ref{sequentialCBF}. The single integrator dynamics guarantee this reachability via omnidirectional velocity limits. For more complex dynamics, similar CBFs can be conservatively formulated utilizing turn-to-go profiles that account for the time required to rotate and translate, though a comprehensive reachability analysis for such systems remains outside the scope of this paper.

\subsection{Proposed STL-Constrained RL Framework}
We focus on a model-free RL setting, where the agent learns a policy by interacting with the environment without prior knowledge of the 
system dynamics or reward function. In this setting, the agent aims to maximize the cumulative episodic rewards through trial and error. Our approach is independent of the RL algorithm and can be integrated with policy gradient algorithms, such as Proximal Policy Optimization (PPO) \cite{schulman2017proximalpolicyoptimizationalgorithms}, Soft Actor-Critic (SAC) \cite{haarnoja2018softactorcriticoffpolicymaximum}, or Trust Region Policy Optimization (TRPO) \cite{schulman2017trustregionpolicyoptimization}.

While these RL algorithms are efficient at learning reward-driven behaviors, as discussed in section \ref{subsec:RL_STL}, they do not inherently guarantee the satisfaction of STL tasks during learning. To ensure that the given STL tasks are satisfied during learning, we embed the sequential CBFs defined in (\ref{sequentialCBF}) into RL. After identifying the most time-critical CBF (\ref{critical_CBF}) for a task sequence $\mathcal{S}$ including $k$ targets, we solve the following Quadratic Program (QP) at each time step $t$ to compute a corrective control $u^{CBF}$:
\begin{flalign}
& (u^{CBF}(x, t, u^{RL}(x)), \epsilon) = \argmin \limits_{u, \epsilon} \|u \|^2 + K_{\epsilon} \epsilon && \label{cbf_qp} \\
& \text{s.t.} \ \frac{\partial \bm{b}(x,t)}{\partial x}^T (u + u^{RL}(x)) + \frac{\partial \bm{b}(x,t)}{\partial x^\text{target}}^T u^\text{target} + \frac{\partial \bm{b}(x,t)}{\partial t} && \notag \\
& \quad + \alpha(\bm{b}(x,t))\geq -\epsilon && \notag\\
& \quad \ \|u + u^{RL}(x)\| \leq u_{\max}, && \notag
\end{flalign}
where $u^{RL}(x) \sim \pi_\theta(\cdot | x)$ is the output of the policy network parameterized by $\theta$ at state $x$, $\epsilon \in \mathbb{R}^+$ is a slack variable with sufficiently large weight penalty $K_{\epsilon}\in\mathbb{R^+}$ that ensures the feasibility of the QP, $x^\text{target} \in \mathbb{R}^{nk}$ and $u^\text{target} \in \mathbb{R}^{mk}$ are the stacked vectors of current target states and target inputs, respectively. 
To account for control input bounds, we impose the constraint $\|u + u^{RL}(x)\| \leq u_{\max}$. The slack variable $\epsilon$ relaxes the first constraint to ensure the feasibility of the QP when it cannot be strictly satisfied due to physical constraints (e.g., actuator limitations). The solution to the QP in (\ref{cbf_qp}) enforces the satisfaction of the given STL task with minimal time relaxation and control effort. The following theorem formalizes this property by providing a lower bound on the task violation (temporal relaxation).
\begin{theorem}
Consider an agent with dynamics as in (\ref{agent_dynamics}), and let $\bm{b}(x, t)$ be a sequential CBF as defined in (\ref{critical_CBF}). Then, if $\alpha (\bm{b}) = \gamma \bm{b}$ for some $\gamma > 0$ and $\bm{b}(x_0, 0) \geq 0$, the value of the barrier function along the system trajectory is bounded from below as:
\begin{equation}
    \bm{b}(x, t) \ge -\frac{\epsilon}{\gamma}.
\end{equation}
In other words, the worst-case violation is bounded below by $-\epsilon/\gamma$, and the controller derived from (\ref{cbf_qp}) renders the agent forward invariant within a relaxed set defined by:
\[
\mathcal{C}_\epsilon = \left\{ x \in \mathbb{R}^n \mid \bm{b}(x, t) \ge -\frac{\epsilon}{\gamma} \right\}.
\]
\label{bound_theorem}
\end{theorem}
\noindent
\hspace{-7mm}\begin{proof}
From (\ref{cbf_qp}), along the system trajectories we have:
\begin{equation}
    \dot{\bm{b}}(x, t) + \gamma \bm{b}(x, t) \geq -\epsilon.
\end{equation}
Multiplying both sides by the integrating factor $e^{\gamma t}$:
\begin{equation}
    \dot{\bm{b}}(x, t) \, e^{\gamma t}  + \gamma \bm{b}(x, t) \, e^{\gamma t} \geq -\epsilon \, e^{\gamma t}.
\end{equation}
Applying the product rule on the left-hand side:
\begin{equation}
    \frac{d}{dt}\left(e^{\gamma t} \bm{b}(x, t)\right) \geq -\epsilon \, e^{\gamma t}.
\end{equation}
Integrating both sides from $0$ to $t$ and solving for $\bm{b}(x, t)$ yields:
\begin{equation}
    \bm{b}(x, t) \geq -\frac{\epsilon}{\gamma} + \left(\bm{b}(x_0, 0) + \frac{\epsilon}{\gamma}\right) e^{-\gamma t} \geq -\frac{\epsilon}{\gamma},
\end{equation}
where the last inequality follows from $\bm{b}(x_0, 0) \geq 0$ and $e^{-\gamma t} \geq 0$ for all $t \geq 0$.
\end{proof}

Theorem 1 can be interpreted as an upper bound on the maximum time-window relaxation in STL satisfaction. In other words, if the agent fails to satisfy a task within the original time bound, the temporal violation (relaxation) is bounded and corresponds to at most $\frac{\epsilon}{\gamma}$ time steps.

After finding the corrective control $u^{CBF}$ by solving the QP, the final control executed by the agent at state $x$ and time $t$ is as follows:
\begin{equation}
    \tilde{u}(x, t) = u^{RL}(x) + u^{CBF}(x, t, u^{RL}(x)).
\end{equation}
Here, the corrective control $u^{CBF}$ modifies the unconstrained RL output $u^{RL}$ to ensure that the executed control respects the STL specification.

To improve the generalization of the learned policy, the agent's state should be randomized at the start of each learning episode \cite{zhang2018studyoverfittingdeepreinforcement}. However, the initial states must be chosen such that the given STL task can be satisfied starting from that state. For example, if the agent is initially far away from the target regions, it might not be able to visit all targets within the desired time windows. Here, we formally define the set of feasible initial states from which the given STL task can be satisfied:

\begin{definition} [Feasible State Set] Given an MDP $\mathcal{M} = (\mathcal{X}, \mathcal{U}, p, r)$, an STL specification $\mathnormal{\Phi}$ which incorporates dynamic targets, we define the feasible state set as follows:
\begin{equation}
    \mathcal{X}_{feasible}^{\mathnormal{\Phi}} = \{x \in \mathcal{X} \mid S_{feasible}^{\mathnormal{\Phi}}(x) \neq \emptyset \},
\end{equation}
where $S_{feasible}^{\mathnormal{\Phi}}(x)$ is the set of feasible sequences for the specification $\mathnormal{\Phi}$ starting from state $x$.
\end{definition}

In other words, if a feasible sequence can be found from a given state, then the overall STL task, as in \eqref{overallSTL}, can be satisfied starting from that state. Having a feasible sequence simply means that $\bm{b}_i(x,t)$ is nonnegative for all tasks $i$ in the sequence, which means each task can be achieved on time. The complete learning algorithm is given in Alg. \ref{main_alg}.

\begin{algorithm2e}[!b]
\caption{STL Constrained RL}
\label{main_alg}
\SetKwInOut{Input}{Input}
\SetKwInOut{Output}{Output}
\Input{MDP $\mathcal{M}$, episode horizon $T$, number of training iterations $N$, STL specification $\mathnormal{\Phi} = \mathnormal{\Phi_1} \land \cdots \land \mathnormal{\Phi_k}$}
Initialize empty buffer $\mathcal{B} = \emptyset$ and policy $\theta \leftarrow \theta_0$

\For{$i$ from $0$ to $N$ }{
    Randomly sample $x_0$ from $\mathcal{X}_{feasible}^{\mathnormal{\Phi}}$
    
    Reset environment

    Generate a feasible sequence $\mathcal{S}$ starting from $x_0$ using a variant of the algorithm in \cite{BUYUKKOCAK2024104681}

    \For{$t$ from $0$ to $T$}{
        $u^{RL}_t = \pi_{\theta_i}(x_t)$
        
        \If{$\mathcal{S} \neq \emptyset$}{

            Construct CBFs and determine the most time-critical one $\bm{b}$ via (\ref{sequentialCBF}) and (\ref{critical_CBF})

            $u_t^{CBF} = QP(x_t, u_t^{RL}, \bm{b}) \quad \triangleright$ \: (\ref{cbf_qp})

            $\tilde{u}_t = u_t^{RL} + u _t^{CBF}$
            }      

        \Else{
            $\tilde{u}_t = u^{RL}_t$

        }
    Take action $\tilde{u}_t$, observe $x_{t+1}$ and $r_t$

    Add transition $(x_t, \tilde{u}_t, x_{t+1}, r_t)$ to $\mathcal{B}$

    \If{$\mathcal{S} \neq \emptyset $}{
        \If{$\mathcal{S}(1)$ is hit}{
        
        $\mathcal{S} = \mathcal{S} \setminus \{ \mathcal{S}(1)\}$
        }

        \If{$\mathcal{S}(1)$ has alternative targets }{
        Calculate the CBF values for the alternatives via (\ref{sequentialCBF})
        
        Update $\mathcal{S}$ if any alternative yields a higher CBF value
        }
        
    }
        }

    Update $\theta_{i}$ using $\mathcal{B}$ via the selected RL algorithm to obtain $\theta_{i+1}$
}
\end{algorithm2e}

Algorithm \ref{main_alg} starts by initializing an empty replay buffer and policy network weights (line 1). At the beginning of each learning iteration, a feasible initial state is sampled, and a task sequence is generated (lines 3-5). At every step, the agent selects an action ($u^{RL}$) from its current policy network (line 7). If the current task sequence is nonempty, a quadratic program is solved to find the corrective control ($u^{CBF}$), which is then added to the policy action (lines 8-11); otherwise, the unmodified action is used as the given STL task is already satisfied (line 13). The resulting action is executed, and the transition is saved in the buffer (lines 14-15). When the agent reaches the target region associated with the first task $\mathcal{S}(1)$, it is removed from the sequence (lines 16-18). If the new first task has alternative targets (i.e., involves disjunction), we evaluate the corresponding CBF values and replace $\mathcal{S}(1)$ with the alternative yielding the highest value (lines 19–21).
Finally, the policy network is updated using the experience stored in the buffer via the selected RL algorithm (line 22).

\section{CASE STUDIES}
We validate the proposed framework through two case studies in a simulation environment. For the learning algorithm, we employ Soft Actor-Critic (SAC) \cite{haarnoja2018softactorcriticoffpolicymaximum}. The policy (actor) and value (critic) networks are feed-forward neural networks consisting of 4 hidden layers of size 256. Each learning episode has a finite length and terminates only when the horizon is reached. To improve generalization, the agent's state is randomized within the feasible state set between episodes. The implementation is done in Python 3.10 and trained on an NVIDIA RTX 4070 GPU.

\begin{figure}[!b]
    \centering
    \begin{subfigure}[b]{0.495\linewidth}
        \centering
        \includegraphics[width=\linewidth]{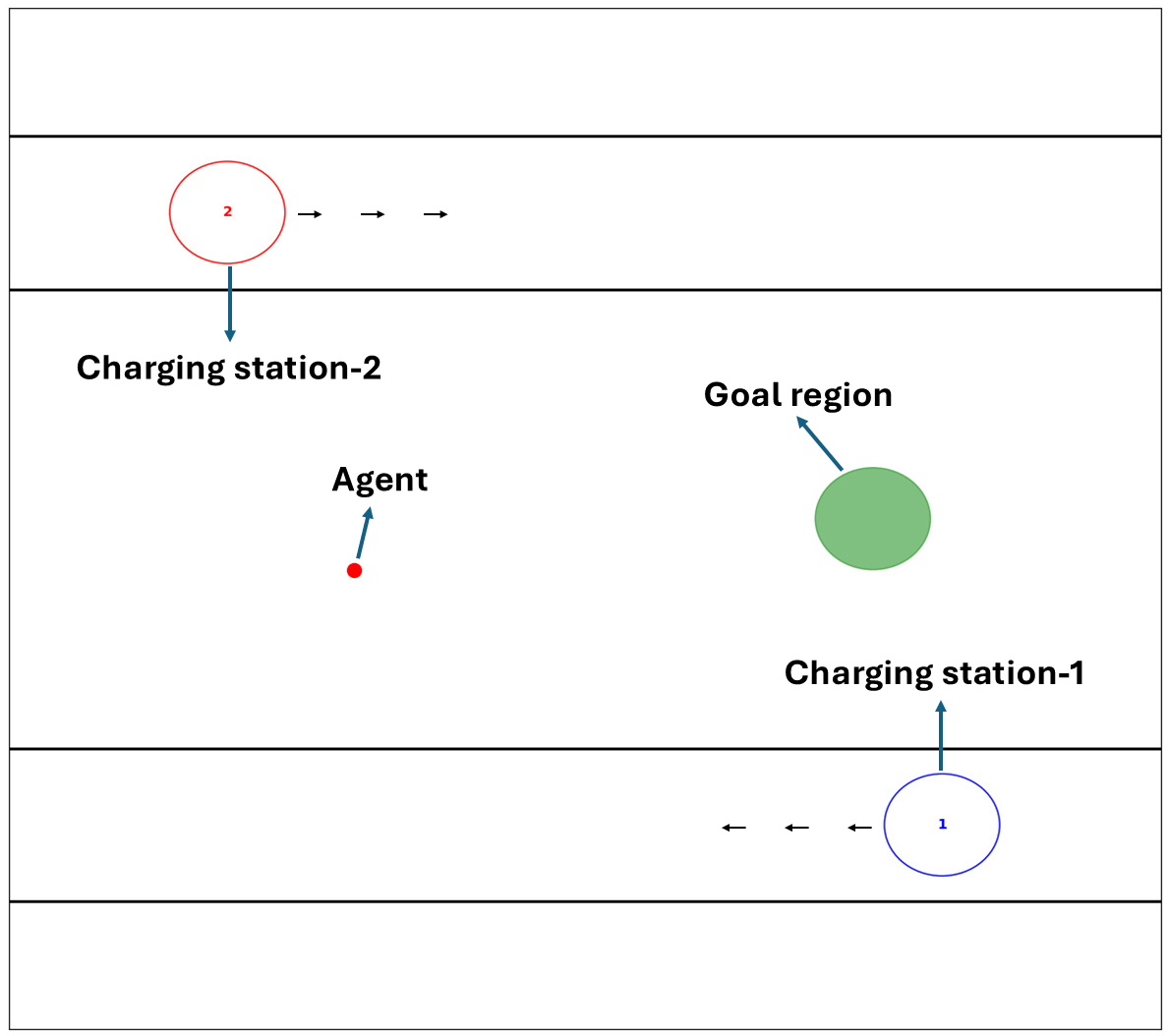}
        \captionsetup{font=footnotesize, labelfont=footnotesize}
        \caption{}
        \label{fig:case1_simulation_env}
    \end{subfigure}
    \hspace{-0.2cm}
    \begin{subfigure}[b]{0.495\linewidth}
        \centering
        \includegraphics[width=\linewidth]{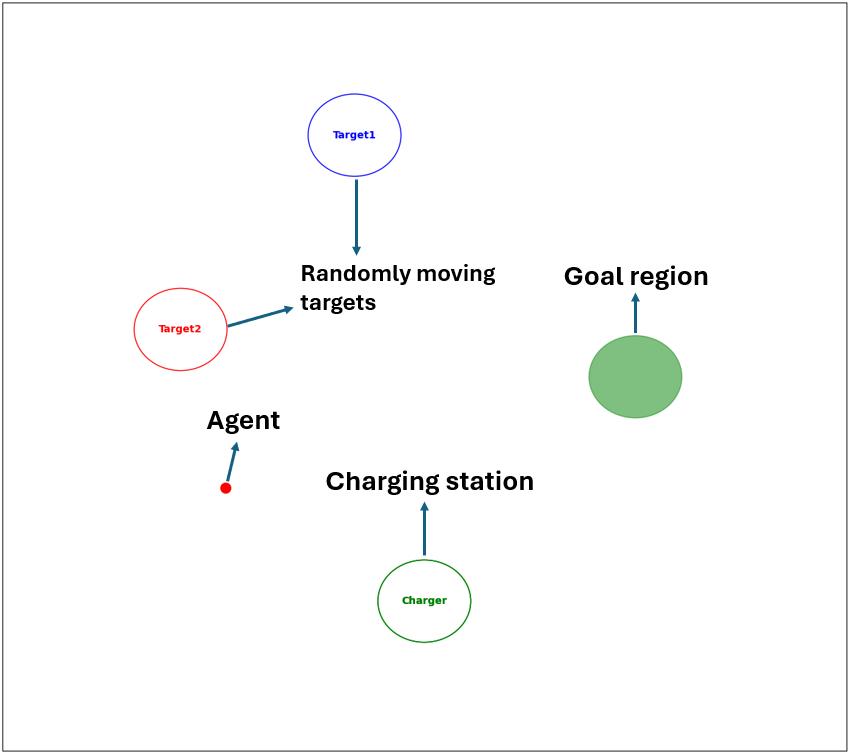}
        \captionsetup{font=footnotesize, labelfont=footnotesize}
        \caption{}
        \label{fig:case2_simulation_env}
    \end{subfigure}
    \captionsetup{font=footnotesize, labelfont=footnotesize}
    \caption{Simulation environment setups: (a) Case 1 and (b) Case 2.}
    \label{fig:simulation_envs}
\end{figure}

In all case studies, the agent follows the dynamics in (\ref{agent_dynamics}) and operates under bounded and uniformly distributed disturbance. Its primary learning objective is to reach a predefined goal region and monitor it under unknown disturbance. Note that the future trajectories of the targets are unknown to the agent. Instead, it can only access instantaneous positions and velocities of the targets, which can be obtained through sensor-based methods.

\textbf{Case 1.} We consider a scenario where the agent must periodically recharge its battery during training. Figure \ref{fig:case1_simulation_env} illustrates the simulation environment used for this case. Two mobile charging stations move back and forth along their respective roads according to the dynamics in \ref{target_dynamics}, and the agent must visit either one of them periodically in each learning episode. We express this task with the following STL specification:
\begin{equation}
    \varPhi_1 = \LTLALWAYS_{[0,210]} \LTLEVENTUALLY_{[0,90]}(Charger_1 \orltl Charger_2).
\end{equation}
This specification enforces that, within each learning episode of length 300 time steps, the agent must visit at least one of the chargers within every window of 90 time steps. Due to the sequence update mechanism in Alg. \ref{main_alg} (lines 19-21), the agent always navigates to the nearest charging station as it yields a higher CBF value. 

Figure \ref{fig:case_learning_curves} compares the learning curves of SAC networks trained with and without the STL task $\varPhi_1$. The unconstrained SAC agent achieves higher episodic rewards as it can focus solely on reward maximization without considering temporal logic tasks. On the other hand, the agent trained under the STL task $\varPhi_1$ converges to a lower episodic reward since it must periodically leave the high-reward area to satisfy the STL task. 
Furthermore, the constrained agent achieved a 99.56\% satisfaction rate of the STL task, which demonstrates that our framework was able to enforce the given STL task effectively throughout the entire learning process. 
\vspace{-2mm}
\begin{remark}
   Rare failure cases occur when the agent reaches the charging stations slightly after the specified time windows, which is effectively extending the deadline instead of failing the task. This is due to the fact that the CBF condition in (\ref{cbf_qp}) is relaxed by a slack variable when input bounds prevent satisfying the inequality. However, Theorem \ref{bound_theorem} provides a lower bound on the resulting time relaxation.
\end{remark}
\vspace{-2mm}
\begin{figure}[!b]
    \centering
    \includegraphics[width=0.9\linewidth]{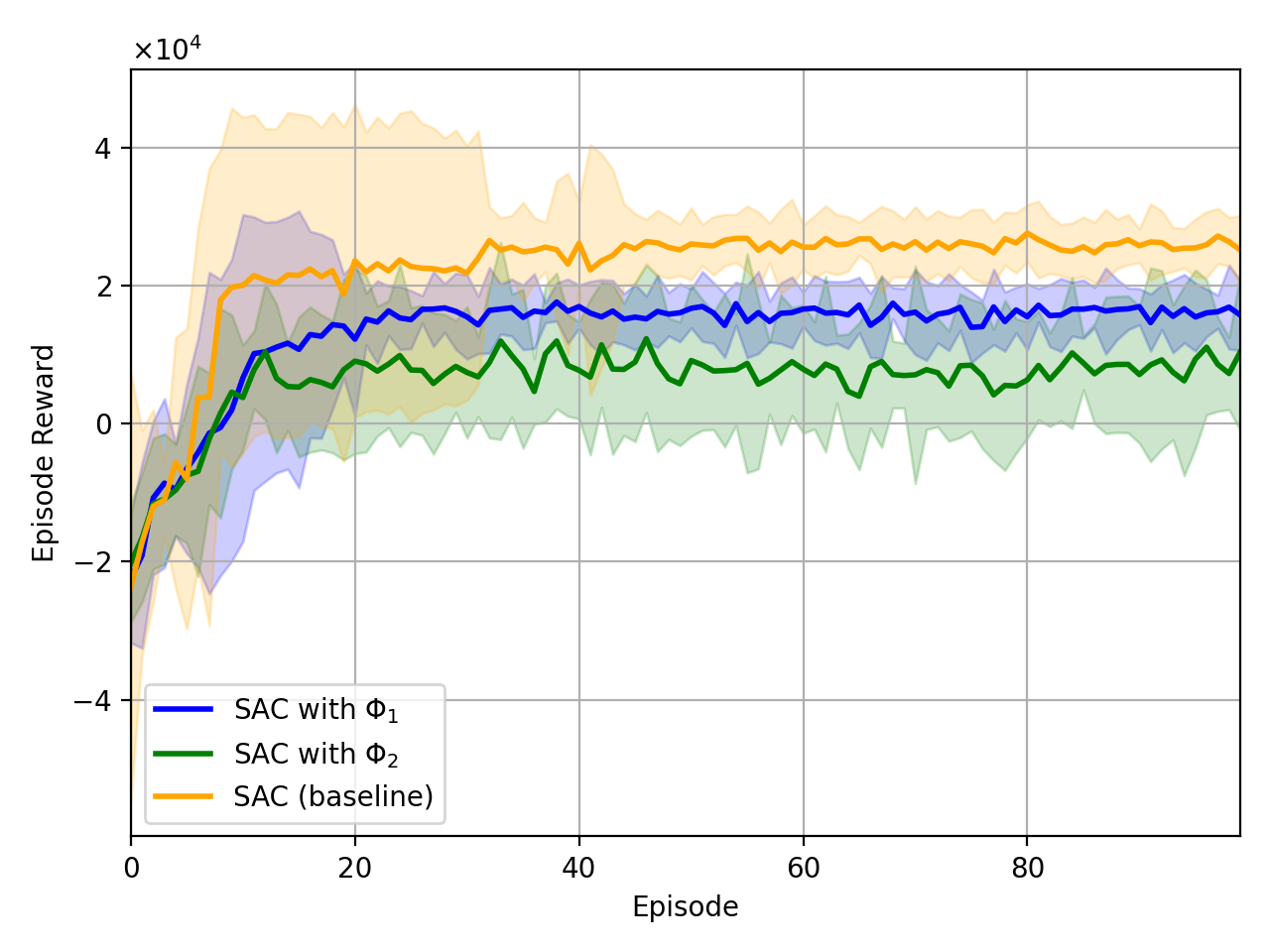}
    \captionsetup{font=footnotesize, labelfont=footnotesize}
    \caption{Learning curve comparisons of SAC models trained without STL tasks (orange) and with the STL specifications $\varPhi_1$ (blue) and $\varPhi_2$ (green). Results are averaged over 10 independent training runs, and the shaded regions represent the confidence intervals of two standard deviations.}
    \label{fig:case_learning_curves}
\end{figure}

\textbf{Case 2.} In this case, we consider a more complex STL task where the agent needs to navigate to two different dynamic target regions within specific time windows and remain inside each for specific durations, while also visiting a dynamic charging station periodically. The STL specification that expresses this task is given as follows:
\begin{flalign}
    \varPhi_2 = & \LTLEVENTUALLY_{[0,60]} \LTLALWAYS_{[0,10]} Target_1 \andltl \LTLEVENTUALLY_{[150,180]} \LTLALWAYS_{[0,10]} Target_2 \notag \\ 
    &\andltl \LTLALWAYS_{[0,190]} \LTLEVENTUALLY_{[0,110]}Charger.
\end{flalign}
This specification requires the agent to visit $Target_1$ within the first 60 time steps and remain there for 10 consecutive steps. Then, it must visit $Target_2$ between time steps 150 and 180 and remain there for 10 steps. Additionally, it must visit the charging station every 110 time steps within each learning episode of length 300 time steps. Both $Target_1$ and $Target_2$ move randomly inside the simulation environment, while the charging station moves in a circular trajectory. Note that neither the targets' nor the charger's future trajectories are known to the agent. 

As shown in Fig. \ref{fig:case_learning_curves}, the learning curve of the SAC agent trained under the STL task $\varPhi_2$ converges to an optimal policy. Since this task requires more visits to target regions and two targets are moving randomly, the agent needs to spend a larger portion of the learning episodes on satisfying the STL task. Consequently, the learning curve for this case converges to a lower episodic reward compared to the unconstrained model and the model trained under $\varPhi_1$. 
Furthermore, the specification was satisfied with a rate of 95.2\% during the learning process. The success rate is slightly lower compared to Case 1, since $\varPhi_2$ imposes stricter deadlines with more sequential requirements, and the randomly moving targets can spread apart significantly, which forces the agent to travel longer distances within limited time windows. Similar to Case 1, the violations here are mainly satisfying the STL task with bounded temporal relaxation. A video showcasing the training processes for both Case 1 and Case 2 is available at: \url{https://youtu.be/rYU4wgfgLyA}.

\section{CONCLUSION}
We introduce a framework that extends sequential control barrier functions to RL settings to enforce complex STL tasks during learning. Unlike prior work that primarily focused on safety, our framework handles richer spatio-temporal specifications, including visits to dynamic targets with unknown trajectories. By formulating sequential CBFs under bounded disturbances and input limits, our method ensures STL satisfaction throughout the learning process rather than only after training. We also provide a theorem that quantifies a lower bound on task violation, which guarantees the achievement of the desired STL task under bounded delays. 
We validate the framework in simulation with bounded disturbances, where the agent learns its primary objective while satisfying STL tasks such as periodic visits to moving charging stations and timed visits to randomly moving targets. 
Future work will extend the framework to more general settings with more complex objectives and agent dynamics.

\bibliographystyle{IEEEtran}
\bibliography{references}

\end{document}